\documentclass{llncs}
\usepackage{llncsdoc}
\usepackage{graphicx}

\usepackage{xcolor}
\usepackage{amsmath}
\usepackage{url}
\usepackage{colortbl}
\usepackage{multirow}

\usepackage{amssymb}

\newcommand{\Var}[0]{\mbox{\texttt{Var}}}

\newcommand{\keywords}[1]{\par\addvspace\baselineskip
\noindent\keywordname\enspace\ignorespaces#1}
\def\code#1{\texttt{#1}}

\begin{document}

\long\def\comment#1{}

\title{Logical conditional preference theories}

\urldef{\mailP}\path|{cornelio, loreggia}@math.unipd.it|  
\urldef{\mailI} \path|vijay@saraswat.org|
\author{Cristina Cornelio\inst{1} \and Andrea Loreggia\inst{1}\inst{2} \and Vijay Saraswat\inst{2} } 
\institute{Padua University, Padua, Italy \\ \mailP \and IBM T.J. Watson Research Center, Yorktown Heights, New York, U.S.\\ \mailI} 
\maketitle

\begin{abstract}
CP-nets represent the dominant existing framework for expressing
qualitative conditional preferences between alternatives, and are used in a
variety of areas including constraint solving. Over the last fifteen
years, a significant literature has developed exploring semantics,
algorithms, implementation and use of CP-nets. 

This paper introduces a comprehensive new framework for conditional
preferences: {\em logical conditional preference theories} (LCP
theories). To express preferences, the user specifies arbitrary
(constraint) Datalog programs over a binary ordering relation on
outcomes.  We show how LCP theories unify and generalize existing
conditional preference proposals, and leverage the rich semantic,
algorithmic and implementation frameworks of Datalog.
\end{abstract}

\section{Introduction}
Qualitative conditional preferences on combinatorial domains are an
important area of study. The main framework is {\em CP-nets} \cite{boutilier2004cp}.  
Given a finite set of features, each with a finite domain of
values, the user specifies her preference order over the domain of
each feature using rules such as $a\, b: c \succ \bar{c}$ which
specify that if the attribute $A$ has value $a$ and attribute $B$ has
value $b$ then $c$ is to be preferred to $\bar{c}$ for attribute $C$. 
{\em Outcomes} (that is, assignment of values to all the features) are ordered according to the
so-called {\em ceteris paribus} interpretation: an outcome $s$ is
preferred to another outcome $t$ if they differ on the value of just
one feature, and the value in $s$ is preferred to the
one in $t$ (given the rest of $s,t$). Thus in CP-nets, preferences are
always of the form ``I prefer fish to meat, {\em all else being
  equal}''. The key computational tasks are checking for consistency
(preference ordering is acyclic), as well as optimizing and comparing
outcomes. 

CP-nets represent one end of the expressiveness/tractability
spectrum. While useful in practice, CP-nets are limited in
expressiveness: a rule may specify a preference for exactly one value
over another (in the same feature domain). {\em General CP-nets}
\cite{goldsmith2008computational} define CP-nets that
can be incomplete or locally
inconsistent. {\em CP-theories} \cite{Wilson04extendingcp-nets} are
 a specialized formalism, with its own {\em ad
  hoc} syntax  (rules are of the form $u:x \succ x'\,[W]$) and
semantics, in which preferences may be conditioned on {\em
  indifference} to certain features (those in the set $W$ above). For
example, one may say ``I prefer fish to meat, no matter what the
dessert is''. {\em Comparative preference
  theories} \cite{wilson09} permit preferences to be defined on a
set of features simultaneously. Along another direction,
\cite{cp+constraints} and \cite{hard+soft+cpnets} introduce the idea
of adding (hard and soft) constraints to CP-nets. Roughly speaking,
the added constraints need to be respected by valid outcomes.


This paper develops the fundamental idea that conditional
preferences can be directly expressed
in standard first order logic, as constrained Datalog
theories \cite{Ceri:1989:YAW:627272.627357,Kanellakis199526,Toman:1998:MEC:593226.593267}
 involving a binary preference relation \code{d(\_,\_)} on pairs
of outcomes. For instance, in the context of preferences over entrees
and desserts, the rule ``I prefer fish to meat, all else being equal''
may be written as:\footnote{The examples in this paper are written in
  Prolog and run in the XSB Prolog system, using tabling. Note
  that the clause is strictly speaking not a Datalog clause because it
  uses a function symbol \code{o/n}. However this function symbol is
  used just for convenience, and can be eliminated at the cost of
  increasing the arity of predicates such as \code{d} and \code{dom}.}
{\footnotesize
\begin{verbatim}
  d(o(fish,X),o(meat,X)).
\end{verbatim}}
\noindent and the rule ``I prefer fish to meat, no matter what the dessert is''
is written as:
{\footnotesize
\begin{verbatim}
  d(o(fish,X),o(meat,Y)).
\end{verbatim}}
\noindent More generally, a
{\em Logical Conditional Preference} (LCP) rule is of the form:
{\footnotesize
\begin{verbatim}
  d(o(X1,...Xn),o(Y1,...,Yn)) :- c, g1,...,gn.
\end{verbatim}}
\noindent where \code{c} is a constraint (possibly involving equalities), and \code{g1}, \ldots, \code{gn} are possibly
recursively defined predicates involving \code{d/2}. A predicate
\code{dom/2} is defined to be the transitive closure of \code{d/2} and
represents the dominance relation on outcomes.
{\footnotesize
\begin{verbatim}
  dom(X,Z):- d(X,Z), outcome(X), outcome(Z).
  dom(X,Z):- d(X,Y), dom(Y,Z).
\end{verbatim}}
\noindent The theory can be checked for consistency by simply ensuring that given the clause
{\footnotesize
\begin{verbatim}
  inconsistent :- dom(X,X).
\end{verbatim}}
\noindent the goal \code{inconsistent} cannot be established. Hard constraints
\code{C} are specified by adding them to the body of the clause
defining legal outcomes: 
{\footnotesize
\begin{verbatim}
  outcome(o(X1,...,Xn)):- C, d1(X1), ..., dn(Xn).
\end{verbatim}}

The fundamental advantage of introducing conditional preference
theories as Datalog programs is that Datalog's rich semantic, algorithmic
and implementation framework is now available in service of
conditional preferences. The semantics of LCP theories is that of
(constrained) first-order logic theories. The framework is rich enough to express
CP-nets and each of its extensions discussed above, including
algorithms for consistency, dominance and optimality
(Section~\ref{sec:CP-extensions}). Using outcomes
rather than assertions of preference over individual features
permits the formalization of the semantics (e.g. {\em
ceteris paribus} or indifference) internally, as just a certain 
pattern of quantification over variables. 

Constraints fit in naturally and do not have to be introduced after the
fact in an {\em ad hoc} fashion as in
\cite{hard+soft+cpnets,cp+constraints,Prestwich04constrainedcpnets}.
For example, \cite{Prestwich04constrainedcpnets} provides a dominance
algorithm using the notion of a {\em consistent flipping worsening
  sequence}: they allow worsening flips only between consistent
outcomes. In our formulation the constraints are already built 
into basic definitions (constraints are additional goals in the body
of preference clauses, and in the body of the clause defining outcomes)
and no changes are necessary.

Additionally, recursive LCP-rules (rules with goals \code{gi} in the
body) offer a powerful new form of {\em dependent} conditional
preference statements (Section~\ref{sec:dependent}), particularly
useful in multi-agent contexts
\cite{rossi2004mcp,MaranInfluencedCPNets}. They support rules such as 
``If Alice prefers to drive to Oxford today, Bob will prefer to fly
 to Manchester tomorrow''.  

The rich complexity theory developed
for Datalog \cite{vardi,Feder:1999:CSM:298483.298498,Gottlob:2003:CSD:794095.794102,Dantsin:2001:CEP:502807.502810} applies {\em inter alia} to conditional
preference theories -- in particular we discuss the notion of
{\em data-complexity} in Section~\ref{bg:datalog-tabled}. General results
about (linear) Datalog programs lead to complexity bounds for
consistency, dominance and optimization extending current known
bounds, and in some cases, providing new, simpler proofs for existing
complexity bounds (Section~\ref{sec:alg}). Further, tabled
Prolog systems such as XSB Prolog\cite{Swift:2010:TAS:1888743.1888771,tabling-xsb} implement
constrained Datalog with sophisticated features such as partial order
answer subsumption that are directly usable in an implementation of
LCP. 

One of the reasons that CP-nets are popular in practice is that
useful special cases have been identified (acyclic nets,
tree-structured nets) which can be implemented efficiently. In
Section~\ref{sec:alg} we show how some of these special cases can be extended
to the richer language we consider. Further, we provide a compiler
for LCP theories that can recognize these special cases and generate
custom code for consistency, dominance and optimization
(Section~\ref{sec:implementation}). We present some scalability numbers.

In summary, we believe our formalization of extensions of CP-nets 
permits an integrated treatment of preferences in constraint (logic) 
programming, leading to more powerful reasoning systems which can deal
with both preferences and hard constraints. 

\paragraph{Rest of the paper.} Section \ref{sec:bg} introduces the
models of GCP-nets, CP-nets, CP-theories and comparative preference
languages, and Datalog and tabled logic programming, and provides basic computational results.
Section \ref{sec:logic} specifies the LCP formulation, and establishes
that LCP-theories conservatively extend GCP-nets, CP-nets, CP-theories
and comparative preference languages.  Section \ref{sec:alg} studies
the computational complexity of outcome optimization, consistency
checking and dominance queries. We also provide specific algorithms for the special cases
(e.g.{} acyclic structure and tree structure dependency graphs). 
Section \ref{sec:implementation} describes the implementation of the
LCP compiler. We conclude in Section \ref{sec:conclusion}.


\section{Background}\label{sec:bg}
%

Below we assume given a set of $N$ {\em features} (or variables), 
$\Var{} = \{X_1,\ldots,X_N\}$. We assume for simplicity that the values in each
feature $X$ are either $x, \bar{x}$ (handling multiple values is easy).

\subsection{GCP-nets and CP-nets}
GCP-nets (\cite{goldsmith2008computational} and  \cite{domshlak03reasoning}) 
allow  a general form of conditional and qualitative
preferences to be modeled compactly.  
\begin{definition}
A \textbf{Generalized CP-net} (\textbf{GCP-net}) $C$ over 
$\Var{}$ is a set of conditional preference rules. A
\textbf{conditional preference rule} is an expression $p : l >  \bar{l}$, where $l$ is
a literal of some atom $X \in \Var{}$ and $p$ is a propositional
formula over $\Var{}$ that does not involve variable $X$. 
A GCP-net corresponds to a directed graph (dependency graph) where
each node is associated with a feature
and the edges are pairs $(Y,X)$ where $Y$ appears in $p$ in
some rule $p:x > \bar{x}$ or $p: \bar{ x}> x$. Each node $X$ is
associated with a \emph{CP-table} which expresses the user preference
over the values of $X$. Each row of the CP-table corresponds to a
conditional preference rule. 
\end{definition}

The CP-tables of a GCP-net can be
\emph{incomplete} (i.e.{} for some values of some variables' parents,
the preferred value of $X$ may not be specified) and/or \emph{locally
  inconsistent} (i.e. for some values of some variables' parents, the
table may both contain the information $x> \bar{x}$ and
$\bar{x}>x$). CP-nets \cite{boutilier2004cp} are a special case of GCP-net in which the
preferences are locally consistent and locally complete. 


%

An \emph{outcome} in a CP-net is a complete assignment to all features.
For example, given  $\Var{}= \{X_1,X_2\}$ and  binary domains
$D_1=D_2=\{T,F\}$, all the possible outcomes are $TT$, $TF$, $FT$ and
$FF$. 

A {\em worsening flip} is a change in the value of a feature to a value which is less preferred 
according to the cp-statement for that feature. This concept defines
an order over the set of outcomes such that one outcome $o$ is {\it
  preferred} to another outcome $o'$ ($o \succ o'$) if and only if
there is a chain of worsening flips from $o$ to $o'$. The notion of
{\em worsening flip} induces a preorder over the set of outcomes. This
preorder allows maximal elements that correspond to the so-called
\emph{optimal outcomes}, which are outcomes that have no other outcome
better than them. 

Given any GCP-net and CP-net the problems of consistency checking 
and finding optimal outcomes are  PSPACE-complete
\cite{goldsmith2008computational}. Moreover, there
could be several different maximal elements. When the dependency graph
has no cycle the CP-net is called {\em acyclic}. The
optimal outcomes for such nets are unique and can be found in
polynomial time in $N$. The procedure used to this purpose is usually
called a {\em sweep forward} and takes $N$ steps \cite{boutilier2004cp}. 

The problem of dominance testing (i.e. determining if one outcome is preferred to another)
is PSPACE-complete for both GCP-nets and CP-nets. It is polynomial
if the CP-nets are tree structured or poly-tree structured 
\cite{domshlak2002cp,goldsmith2008computational}. 

\subsection{CP-theories and comparative preference languages}\label{wilson}

CP-theories are introduced in \cite{Wilson04extendingcp-nets}  as a logic of conditional preference 
which generalizes CP-nets.

\begin{definition}
Given a set of variables $\Var{}=\{X_1, \cdots , X_N\}$ with domains $D_i$, $i=1,\ldots,n$,
the language $L_{\Var{}}$ is defined by all the statements of the form: $ u: x_i \succ x'_i [W]$ where $u$ is an assignment of a set of variables $U \subseteq \Var{} \setminus \{X_i\}$, $x_i \not = x'_i \in D_i $ and $W$ is a set of variables such that $W \subseteq (\Var{} \setminus U  \setminus \{X_i\})$.
\end{definition}

\begin{definition}
Given a language $L_{\Var{}}$ as defined above, a \emph{conditional preference theory (CP-theory)} $\Gamma$ on $\Var{}$ is a subset of $L_{\Var{}}$. $\Gamma$ generates a set of preferences that corresponds to the set $\Gamma^*= \bigcup_{\varphi \in \Gamma} \varphi^*$ where given $\varphi=u: x_i \succ x'_i [W]$, $\varphi^*$ is defined as $\varphi^*=\{(tuxw,tux'w'): t \in \Var{} \setminus (\{X_i \cup U \cup W\}), ~ w, w' \in W\}$.
\end{definition}

A CP-net is a particular case of a CP-theory where $W= \emptyset$ for all $\varphi \in \Gamma$.

Two graphs are associated to a CP-theory:
$H(\Gamma)=\{(X_j, X_i)| \exists \varphi \in \Gamma \text{ s.t. }  \varphi=u: x_i \succ x'_i [W] \text{ and } X_j \in U\}$
and $G(\Gamma)=H \cup \{(X_i, X_j)| \exists \varphi \in \Gamma  \text{ s.t. } \varphi=u: x_i \succ x'_i [W] \text{ and } X_j \in W\}$.

The semantics of CP-theories depends on the notion of a
\emph{worsening swap}, which is a change in the assignment of a set of
variables to an assignment which is less preferred by a rule $\varphi
\in \Gamma$. We say that one outcome $o$ is better than another
outcome $o'$ ($o \succ o'$) if and only if there is a chain of
worsening swaps (a \emph{worsening swapping sequence}) from $o$ to
$o'$. 

\begin{definition}
A CP-theory $\Gamma$ is {\em locally consistent} if and only if  for all
$X_i \in \Var{}$ and $u \in Pa(X_i)$ in the graph $H(\Gamma)$, $\succ_{u}^{X_i}$ is irreflexive. 
\end{definition}

Local consistency can be determined in time proportional to
$|\Gamma|^2N$. Given a CP-theory $\Gamma$, if the graph
$G(\Gamma)$ is acyclic, $\Gamma$ is consistent if and only if
$\Gamma$ is locally consistent, thus global consistency has the same
complexity as local consistency given an acyclic graph $G(\Gamma)$. 


Comparative preference theories \cite{wilson09} are an extension of CP-theories. 
\begin{definition}
The comparative preference language ${\cal CL}_{\Var{}}$ is defined by
all statements of the form: $ p>q ||T$ where $P$, $Q$ and $T$ are
subsets of $\Var{}$ and $p$ and $q$ are assignments respectively of
the variables in $P$ and in $Q$. 
\end{definition}
\begin{definition}
Given a language ${\cal CL}_{\Var{}}$ as defined above, a
\emph{comparative preference theory} $\Lambda$ on $\Var{}$ is a subset
of ${\cal CL}_{\Var{}}$. $\Lambda$ generates a set of preferences that
corresponds to the set $\Lambda^*= \bigcup_{\varphi \in \Lambda}
\varphi^*$ where if $\varphi= p>q ||T$, $\varphi^*$ is defined as a
pair $(\alpha,\beta)$ of outcomes such that $\alpha$ extends $p$ and
$\beta$ extends $q$ and $\alpha$ and $\beta$ agree on $T$: $\alpha
\restriction _{T}=\beta \restriction _{T}$. 
\end{definition}

\subsection{Datalog and tabled logic programming}\label{bg:datalog-tabled}
A Datalog program consists of a collection of definite clauses in a
language with no function symbols, hence a finite Herbrand
domain. Datalog programs can be implemented using {\em tabled Logic
  Programming (TLP)}. Tabling 
maintains a memo table of subgoals  produced in a query
evaluation and their answers. If a subgoal is reached again then the
information in the table can be reused, without recomputing the
subgoal. This method ensures termination and improves the
computational complexity for a large class of problems
\cite{tabling-xsb} (at the expense of additional space
consumption). Answer subsumption extends the functionality of tabling.
{\em Answer variance} adds a new answer to a table only if the new
answer is not a variant of any other answer already in the table. {\em
  Partial order answer subsumption} adds a new answer to a table only 
if the new answer is maximal with respect to the answers in the table,
given a partial order \code{po/2}:

{\footnotesize \begin{verbatim}
:-  table predicate(_,_,partialOrder(po/2)).
\end{verbatim}
}

Traditionally, predicates are divided into {\em extensional}
and {\em intensional} predicates. The extensional predicates define a
database, and intensional predicates define (possibly recursively
defined) queries over the database. In our context, \code{d/2} and \code{outcome/1} will be
considered extensional predicates (in Flat LCP) and other predicates
such as \code{dom/2, inconsistent} and user-defined predicates are
considered intensional. Given an intensional program $P$, database $D$
and query $q$, {\em data-complexity} \cite{vardi} is the complexity of answering
$P,D\vdash q$ as a function of the size of $D$ and $q$ (thus the
program is considered fixed). {\em Combined complexity} is the
complexity of answering 
$P,D\vdash q$ as a function of the size of $P$, $D$ and $q$ (thus
nothing is taken to be fixed). 
The basic results are that for data-complexity, general Datalog
programs are {\sc PTIME}-complete, and linear programs are {\sc NLOGSPACE}-complete
\cite{Gottlob:2003:CSD:794095.794102}. For combined complexity,
general Datalog programs are {\sc EXPTIME}-complete, and linear programs are
{\sc PSPACE}-complete. 


\section{Logical Conditional Preference Theories}\label{sec:logic}
We assume given a set of $N$ features, and a logical vocabulary $\cal V$ with
unary predicates \code{d1}, \ldots, \code{dN} (corresponding to
the domains of the features), constants for every  
value in the domains \code{di}, a single function symbol 
\code{o/N}, and a single binary predicate \code{d/2}.  

The user specifies preferences between two outcomes 
\code{S,T} (expressed as \code{o/N} terms)
by supplying  clauses for the atom
\code{d(S,T)}:
{\footnotesize\begin{verbatim}
d(S,T) :- C, g1, ..., gk.
\end{verbatim}}
\noindent where \code{C} is a constraint (possibly involving
equality), and \code{g1}, \ldots, \code{gk} are possibly 
recursively defined predicates involving \code{d/2}. The  clause is
said to be {\em flat} if \code{k=0}, else it is {\em recursive}. The
user specifies hard constraints on features by providing clauses for the
\code{outcome/1} predicate, typically of the form
{\footnotesize
\begin{verbatim}
outcome(o(X1,...,Xn)) :- C, d1(X1), ..., dn(Xn).
\end{verbatim}
}
\noindent where \code{C} is a constraint and the \code{di} are domain predicates.

The LCP runtime supplies the following definition for the \code{dom/2}
predicate, expressing (tabled) transitive closure over \code{d/2}, and
for consistency and optimal outcomes: 

{\footnotesize
\begin{verbatim}
:- table(dom(_,_)).
dom(X,Y):- d(X,Y), outcome(X), outcome(Y).
dom(X,Y):- d(X,Z), dom(Z,Y).
consistent :- \+ dom(X,X).
:- table(optimal(po(dom/2))).
optimal(X):- outcome(X).
\end{verbatim}
}
\noindent Note that the clauses above are linear. Below, given an LCP
theory (Datalog program) $P$, by ${\cal L}(P)$ we will mean $P$
together with the LCP runtime clauses specified above.

Given these definitions, the problem solver may use
\code{consistent} to determine whether the supplied preference
clauses are consistent, \code{dom(S,T)} to determine whether
outcome \code{S} is preferred to \code{T}, and 
\code{optimal(S)} (where \code{S} may be a partially
instantiated \code{o/N} structure) to determine an optimal
completion of \code{S}.

\begin{example}[Dinner, modified from \cite{boutilier2004cp}]\label{ex:Dinner}
Two components of a meal are the soup (\code{fish} or \code{veg})
and wine (\code{white} or \code{red}). I prefer \code{fish} to
\code{veg}. If I am having \code{fish}, I prefer \code{white} wine to
\code{red}. I simply do not want to consider \code{veg} with 
\code{red}. This may be formulated as the LCP theory:
{\footnotesize
\begin{verbatim}
soup(fish). soup(veg). wine(white). wine(red).
outcome(o(X,Y)):- soup(X), wine(Y), (X\== veg; Y\==red).
d(o(fish, X),    o(veg,  X)).
d(o(fish, white),o(fish, red)).
\end{verbatim}}
\noindent On this theory, the query \code{?-consistent.} returns
\code{yes}. The \code{dom/2} predicates order outcomes as:
{\footnotesize
\begin{verbatim}
o(fish,white) > o(fish,red)   o(fish,white) >  o(veg,white)
\end{verbatim}}
\noindent Note because of hard constraints the outcomes are not
totally ordered. The query \code{optimal(X)} returns the single answer
\code{X=o(fish,white)}; the query \code{optimal(o(fish,X))} returns \code{X=white}; 
\code{optimal(o(X,red))} returns \code{X=fish}, etc. The behavior of
the \code{optimal/1} queries will be explained later; for now observe
that an \code{optimal(O)} query returns that (in this case, unique)
instantiation of \code{O} which is highest in the \code{dom/2} order.
\end{example}

\begin{example}[Holiday Planning, \cite{Wilson04extendingcp-nets}]\label{ex:HolidayPlanning}
  There are three features: \code{time}, with values \code{l} and
  \code{n} for later and now; \code{place}, with
  values \code{m} and \code{o} for Manchester and Oxford, and \code{mode}
  with values \code{f} and \code{d} for fly and drive. 
  The preference ``All else being equal, I would prefer to go to
  Manchester'' is formulated as clause 1 below.  The rule ``I would
  prefer to fly rather than drive, unless I go later in the year to
  Manchester, where the weather will be warmer, and a car would be
  useful for touring around'' translates to clauses 2-3.
  The CP-theory rule ``I would prefer to go next week, regardless of other
    choices.'' corresponds translated to clause 4, and the comparative
    preference rule ``All other things being equal, I would prefer to
    fly now, rather than to drive later.'' to clause 5:

{\footnotesize
\begin{verbatim}
time(n). time(l). place(o). place(m). mode(f). mode(d).
outcome(o(T,P,M)):- time(T), place(P), mode(M).
/*1*/ d(o(X,m,Y),o(X,o,Y)).
/*2*/ d(o(T,P,f),o(T,P,d)):- T=n;P=o.
/*3*/ d(o(l,m,d),o(l,m,f)).
/*4*/ d(o(n,_,_),o(l,_,_)).
/*5*/ d(o(n,X,f),o(l,X,d)).
\end{verbatim}}
\end{example}

  
\begin{proposition}[Normal form for Flat LCP rules]
Let $R$ be a flat LCP-rule \code{d(o(}$X_1, \ldots, X_n$\code{), o(}$Y_1, \ldots ,Y_n$\code{)) :- c.}
where \code{c} is an equality constraint. For appropriate choices of
disjoint index sets $J$, $K$, $M$ and $Z$ s.t.  
  $\{1, \ldots, n\}=J \cup K \cup M \cup Z$, and given $L$ and  $U$ disjoint subsets of $M$  and given constants 
  $v_j (j \in J)$, $a_m,a'_m (m\in M, a_m \not= a'_m)$, $a_l (l\in L)$ and $a_u (u\in U)$,
  $R$ is logically equivalent to the clause 
\code{d(o(}$S_1, \ldots, S_n$\code{), o(}$T_1, \ldots ,T_n$\code{)).}
where  $S_i,T_i$ are defined by:  $S_j=T_j= v_j (j \in J)$,
 $S_k=T_k, k \in K$,  $S_m=a_m, T_m=a'_m (m \in M)$, 
$S_z=X_z,T_z=Y_z (z \in Z)$,
$S_l=X_l, T_l = a_l (l \in L)$,
$S_u=a_u, T_u= Y_u (u \in U)$.

The set $J$ corresponds to the parent variables, the set $K$ to the
ceteris paribus variables, the set $M$ to the variables that change
the value from $S$ to $T$ and $Z$ to the variables that are less
important than the variables in $M$. We call $L$ the {\em lower-bound}
set, and $U$ the {\em upper-bound} set. 
\end{proposition}
  
\begin{example}
Given three variables \code{main}, \code{drink} and \code{dessert}
with domains \code{\{meat,} \code{fish,veg\}}, \code{\{water,wine\}} and
\code{\{cake,fruit\}} respectively.  The rule ``If I eat cake as dessert I prefer to drink
  water and I prefer to not eat meat'', gives \code{Z,U,K=$\emptyset$},
  \code{J=\{dessert\}},  \code{M=\{main, drink\}}, and \code{L=\{A\}}
  (clause 1). The rule ``Given the same dessert, I always prefer fish
  as main course, regardless of the drink.'', gives 
  \code{J,L=$\emptyset$} \code{Z=\{drink\}}, \code{K=\{dessert\}}, and
  \code{M=U=\{main\}} (clause 2):
{\footnotesize\begin{verbatim}
/*1*/ d(o(X1, water, cake), o(meat, wine, cake)).
/*2*/ d(o(fish, X2, X3),o(Y1, Y2, X3)).
\end{verbatim}}
\end{example}

A GCP-net rule is simply a Flat LCP-rule such that
$Z=\emptyset$, $|M|=1$, $L\cup U=\emptyset$. A CP-theory
rule is a Flat LCP-rule with $|M|=1$, 
$L \cup U=\emptyset$. A comparative preference rule is
a Flat LCP-rule with $L\cup U=\emptyset$.

The following theorems establish that LCP-theories conservatively
extend these sub-languages. Proofs are straightforward, we have
essentially just used standard logical notions to formalize the sub-languages:
\begin{theorem}[Logical characterization of {\em ceteris paribus and general ceteris paribus}]\label{th:CP-GCP}
Given a CP-net $\cal R$, consider 
the set $P$ of flat LCP-rules  which represent all rows of its CP-tables.
Then, for any two outcomes $s$ and $t$, $s \succ
t$ in $\cal R$ iff 
${\cal L}(P) \vdash \mbox{\code{dom(s,t)}}$.
\end{theorem}
\begin{theorem}[Logical characterization of {\em CP-theories} and {\em comparative preference languages}]\label{th:comparative-CP-TH}
Given a CP-theory ${\cal R}$, consider the set $P$ of flat LCP rules modeling
all the rules of the CP-theory.
Given two outcomes $s$ and $t$,
${\cal R}\vdash s \succ t$ iff  ${\cal L}(P) \vdash \mbox{\code{dom(s,t)}}$.
\end{theorem}

\label{sec:CP-extensions}

\subsection{Recursive LCP-theories}\label{sec:dependent}

Recursive or dependent rules are particularly useful in multi-agent
contexts, where different agents may influence each other 
by stating their preferences depending on the preferences of some other agent
\cite{MaranInfluencedCPNets}.  Here we illustrate with an extension to
the Holiday planning example: 

\begin{example}[Holiday planning]
John and Mary work for the same office and need to travel separately.
``If John prefers Oxford to Manchester (all other things being equal),
then Mary prefers Manchester to Oxford (all other things being equal).''
{\footnotesize\begin{verbatim}
jPlace(X,Y) :- dom(o(J1, X, J3, M1, M2, M3), o(J1, Y, J3, M1, M2, M3)).
d(o(J1, J2, J3, M1, m, M3), o(J1, J2, J3, M1, o, M3)).
\end{verbatim}}

The predicate \code{jPlace(X,Y)} may be read as ``John prefers place
\code{X} to place \code{Y}, all other things being equal.'' The second clause
may be read as saying ``Mary prefers \code{m} to  \code{o}, all
other things being equal, provided that \code{jPlace(o,m)} holds.
\end{example}

\begin{example}
Let us consider two agents ranking features ``appetizer'' (rolls or
bread), ``main dish'' (pasta or fish) and dessert (tiramisu or
bread-pudding). We can formulate ``If Alice doesn't prefer pasta, I
would like to take pasta'' as:
{\footnotesize\begin{verbatim}
d(o(AA, AM, AD, MA, pasta, MD), o(AA, AM, AD, MA, fish, MD)) :- 
   dom(o(_, fish,_,_,_,_),o(_, pasta,_,_,_,_)).
\end{verbatim}}
\end{example}
Note we do not assume acyclicity in variable ordering.

%
%
%
%
%
%

\section{Algorithmic properties}\label{sec:alg}
The main algorithmic tasks regarding a preference theory are 
{\em dominance queries}, {\em consistency checking}, and {\em outcome
optimization} (also of interest are {\em ordering queries}, but we
  rule them out of scope because of limitations of space). Below we fix
  a set of features $\Var$ with cardinality $N$ and an LCP-theory $P$
  over $\Var$. 

Checking the dominance over a pair of outcomes corresponds to finding a
swapping sequence in CP-theories or a flipping sequence in
CP-nets. For LCP-theories this is determined by first-order
derivability: the dominance query \code{(s,t)} succeeds iff ${\cal
  L}(P)\vdash \mbox{\code{dom(s,t)}}$. 

The problem of consistency checking is checking whether there is any
outcome \code{s} such that ${\cal L}(P)\vdash \mbox{\code{dom(s,s)}}$.

The problem of outcome optimization corresponds to find the most
preferred outcome given an assignment to a (possibly empty) subset  of
features.

\begin{definition}[Optimal outcome]
An \emph{optimal outcome} is an outcome \code{s} such that 
there is no other outcome \code{t} such that ${\cal L}(P)\vdash \mbox{\code{dom(s,t)}}$ (i.e.{}
\code{s} is an undominated outcome).  
\end{definition}
\begin{definition}[Optimal completion]
Given a (possibly non-ground) term \code{s} (representing a partial outcome)
an \emph{optimal completion} of \code{s} is a ground term \code{t}
instantiating \code{s} s.t. for no other ground term \code{t1}
instantiating \code{s} is it the case that 
${\cal L}(P)\vdash \mbox{\code{dom(t1,t)}}$.
\end{definition}
Note that an acyclic CP-net and CP-theory has
a unique optimal outcome, but an LCP-theory may have several optimal
outcomes. 


Our algorithmic approach is based on analyzing whether the input
LCP-theory corresponds to a special case (e.g. acyclic CP-net,
tree-structured dependency graph, etc.). If so, optimal algorithms for
the special case are used. Otherwise general Datalog procedures are
used. 

In the following sections we analyze the algorithms for dominance,
consistency and optimality from a general point of view, and then
considering the special cases in which the algorithms are faster.
We take the viewpoint of {\em combined complexity}, i.e.{}
assuming $N$, the clauses of the LPC theory, and the given query are
supplied at runtime.

The results are summarized in Table \ref{table:complexity1}. 
\begin{table}

\begin{center}
\begin{scriptsize}
\begin{tabular}{ccccc}
&&\cellcolor{gray!40}General structure &\cellcolor{gray!40}Acyclic &\cellcolor{gray!40}Tree\\
\hline
{\bf Dominance}&&&&\\
\multirow{2}{*}{CP-nets:} && PSPACE-comp & PSPACE-comp & Polynomial \\
 &&\cite{goldsmith2008computational} &  \cite{goldsmith2008computational} &  \cite{boutilier2004cp,francesi} \\
\multirow{2}{*}{CP-theories:}  && PSPACE-comp & PSPACE-comp & {\bf Polynomial} \\
 &&\cite{Wilson04extendingcp-nets} &  \cite{Wilson04extendingcp-nets} &   \\
Flat LCP-theories: && {\bf PSPACE-comp} &{\bf PSPACE-comp} & {\bf  ? }\\
Recursive LCP-theories: && {\bf EXPTIME-comp} & {\bf PSPACE-comp}& {\bf  ?}\\
\hline
{\bf Consistency/Optimality} &&&&\\
\multirow{2}{*}{CP-nets:} && PSPACE-complete & Polynomial & Polynomial  \\
 &&\cite{goldsmith2008computational} &  \cite{boutilier2004cp} &  \cite{boutilier2004cp} \\
\multirow{2}{*}{CP-theories:} && PSPACE-complete & Polynomial & Polynomial  \\
 &&\cite{Wilson04extendingcp-nets} &  \cite{Wilson04extendingcp-nets} &  \cite{Wilson04extendingcp-nets} \\
Flat LCP-theories: && {\bf PSPACE-complete/?} & {\bf Polynomial*} & {\bf Polynomial *}  \\
Recursive LCP-theories: && {\bf EXPTIME-complete} & {\bf PSPACE-complete }& {\bf ? } \\
\hline
\end{tabular}
\end{scriptsize}
{\em Bold results are provided in this paper, *with some constraint on
  the form of the rule, ? means the corresponding problem is open.}
\caption{Computational complexity of Dominance, Consistency and
  Optimality}

\label{table:complexity1}
\label{table:complexity2}
\end{center}
\end{table}

It is important to notice that for Recursive LCP-theories optimality
and consistency procedures never have a lower computational complexity
then the dominance procedure, because a recursive LCP rule also
contains a dominance query. 

We note in passing that for Flat LCP theories, data-complexity is also
of interest. Recall that data-complexity for a Datalog programs is the
complexity of determining, for a fixed program $P$, and input database $D$
and query $q$, whether $P,D \vdash q$ (as a function of the size of
$D$ and $q$).  

What is the distinction between $P$ and $D$ for LCP theories? For Flat
LCP theories, $P$ is simply the clauses for \code{dom/2},
and \code{consistent/0}. Once $N$, the number of
features is fixed, this program is fixed. 
Thus data complexity for consistency of LCP theories corresponds to
the complexity of determining for {\em fixed $N$}, 
whether $P, D \vdash \mbox{\code{inconsistent}}$, as a function of the
number of rules in the program. For Flat LCP (=linear Datalog) the data-complexity is {\sf
  NLogSpace}-complete (see e.g.{} \cite{Gottlob:2003:CSD:794095.794102}).

\subsection{Dominance}
\begin{theorem}\label{dominance-flat}
  Given a flat LCP theory $P$ over $N$ features, deciding
  \code{dom(s,t)} is {\sc PSPACE}-complete in $N$.
\end{theorem}
\begin{proof}
  Since flat LCP theories can encode the GCP-nets of
  \cite{goldsmith2008computational}, the dominance problem is at
  least PSPACE-hard. That the problem is in PSPACE can be established
  in a form similar to the proof of Theorem 4.4 in
  \cite{Gottlob:2003:CSD:794095.794102}.  Since $P$ has but a single rule, and
  the rule is linear, we can build up the proof non-deterministically
  using a polynomial amount of space. In fact, we need to keep space
  only for two ground facts of the form \code{dom(s, t)}
  where the \code{s,t} are constants. We start by using the base
  clause for \texttt{dom/2} to non-deterministically establish a 
  \texttt{dom/2} fact, using some fact for \code{d/2}, and scratch
  space linear in $N$.  Then we use the recursive clause for
  \code{dom/2} to non-deterministically generate a new
  \code{dom/2} fact. From this new fact, we can generate another, and
  delete the old fact. We stop when \code{dom(s,t)} is established.  
\end{proof}
\begin{theorem}\label{dominance-recursive-lcp}
  Given a recursive LCP theory $P$ over $N$ features, deciding
  \code{dom(s,t)} is {\sc EXPTIME}-complete in $N$.
\end{theorem}
\begin{proof}
  The proof is as above, except that the \code{d/2} clauses may no
  longer be linear, hence the combined complexity for full Datalog
  comes into play.
\end{proof}
We note in passing -- and document in the supplementary material of
this paper -- that the connection with Datalog allows for a simple and
direct proof of the {\sc PSPACE}-hardness of dominance for
CP-nets. We show that the well-known {\sc PSPACE-hard} problem of
determining whether a deterministic Turing Machine can accept the
empty string without ever moving out of the first $k$ tape cells can
be reduced to checking dominance queries for CP-nets by modifying
slightly the proof for Datalog in \cite[Theorem
  4.5]{Gottlob:2003:CSD:794095.794102}.    

%

In practice, the solutions of the dominance problem can be found using
tabling on the \code{dom/2} predicate. 

Note that in tree structured CP-nets a dominance query can be computed in
time linear in $N$ \cite{boutilier2004cp,francesi}. We observe that a procedure similar to 
\cite{francesi} can be used for CP-theories. We use the generalized
dependency graph $G$ described in \cite{Wilson04extendingcp-nets} (see
Section \ref{wilson}, this includes ``importance'' edges and
dependency edges), and when $G$ is a tree, we apply the dominance procedure of
\cite{francesi}.

\subsection{Consistency and Outcome optimization}\label{sec:optimization}

Consistency is determined by invoking theq query 
\code{?-consistent.} This takes advantage of the tabling of the
\code{dom/2} predicate.

The following theorem affirms that
consistency remains in PSPACE even when the language for preferences
is extended beyond CP-nets to flat LCP rules, and it is a generalization of Theorem 3 in
\cite{goldsmith2008computational}.
\begin{theorem}
  Given a flat LCP theory $P$ over $N$ features, deciding
  consistency is {\sc PSPACE}-complete in $N$.
\end{theorem}
The proof follows directly from the proof of Theorem~\ref{dominance-flat}
since consistency is reduced to checking entailment.

\begin{theorem}
  Given a recursive LCP theory $P$ over $N$ features, deciding
  consistency is {\sc EXPTIME}-complete in $N$.
\end{theorem}
 As above, noting that the combined complexity for full Datalog is
 {\sc EXPTIME}.

For optimality, the user invokes the query \code{?-
  optimal(s).}  Note that \code{optimal/1} uses
partial order answer subsumption (see Section \ref{bg:datalog-tabled}).
In theory this may result in an exponential number of
calls to \code{dom/2} atoms, with each check taking exponential
time. This leads to:
\begin{theorem}
  Given a recursive LCP theory $P$ over $N$ features, deciding
  optimality is in {\sc EXPTIME} over $N$.
\end{theorem}
For now we leave as open the corresponding problem for flat recursive
theories (note that this problem is {\sc PSPACE}-complete for CP-nets).

\subsubsection{Optimization and Consistency for acyclic dependency graphs}
In acyclic CP-nets the sweep-forward procedure 
\cite{boutilier2004cp} finds the unique optimal outcome (or completion)
in polynomial time. A similar result holds for 
\cite{Wilson04extendingcp-nets}.

Under the assumptions of consistency and acyclicity,
an optimal outcome (and completion) can also be found for Flat LCP
theories in polynomial time, using the following 
algorithm {\em
  Acyclic-LCP-Opt} 
  generalizing sweep-forward (see Section~\ref{sec:implementation} for
  implementation details):
\begin{enumerate}
\item Given a set of LCP-rules we compute the dependency graph $G$,
  and we check if $G$ is acyclic.  
\item We compute a total order ${\cal O}=\{X_1,\ldots, X_N\}$ over $\Var$ as a
  linearization of  the topological order defined by $G$.
\item For each variable $X_i$, chosen following ${\cal O}$, we
  consider a set $W_i \subseteq \Var{}$ such that it contains all the
  variables that change the value jointly with $X_i$ in at least one
  rule ($W_i$ could contain only $X_i$).  Using the rules in the
  LCP-theory that involve the variables in $W_i$, and given the
  assignments for the variables $X_1,\ldots, X_{i-1}$,  we generate an
  ordering over the partial outcomes defined on the variables in
  $W_i$. We assign to $X_i$ the $X_i$ value of the top element of
  the ordering that satisfies \code{outcome/1} for at least one
  completion (the completion has the given
  assignment to $X_1,\ldots, X_{i-1}$ and an arbitrary assignment to
  $\{X_{i+1},\ldots, X_N\} \setminus W_i$). 
\item We repeat the previous step for all the features in $\Var{}$
  (following ${\cal O}$), never changing the value of an assigned
  variable. 
\end{enumerate}

\begin{proposition}
The outcome obtained with the {\em Acyclic-LCP-Opt} procedure is
an optimal outcome for acyclic LCP-theories.
\end{proposition}
\begin{proof}
We suppose by contradiction that exists an outcome $o'$ such that $o'
\succ o$. This implies that exists a chain of outcomes $o'=o_1,
\cdots, o_m=o$ such that $\forall i \in \{1, \cdots, m-1\}$ exists a
rule $R$ that implies $o_i \succ o_{i+1}$. Considering a linearization
$X_1, \cdots , X_N$ of the graph $G$ associated to our set of rules
${\cal C}$, we take the first variable $X$ in this order such that
$o'\restriction_X \not = o \restriction_X$ . Thus there exist a rule
$R$ and an index $i \in \{1, \cdots, m-1\}$ such that
$o'\restriction_X=o_1\restriction_X= \cdots =o_i\restriction_X \not =
o_{i+1} \restriction_X= \cdots = o_m\restriction_X=o\restriction_X$
and $o_i \succ o_{i+1}$. This implies that  $o\restriction_X$ is not
the maximal element  in the set of rules that involve $X$, that is a
contradiction. 
 \end{proof}

\label{sec:Acyclic-LCP-Opt}
In LCP-theories it is possible to have many different optimal outcomes: 
we can obtain the whole set of optimal outcomes using the {\em Acyclic-LCP-Opt} procedure. 
If there is more then one optimal outcome this means that there exists some variable $X$ that in the third step of the algorithm has more then one top element. 
Running in parallel all these possible assignments we obtain the whole set of optimal outcomes.

The complexity of the procedure is $O(d^{w}*N)$ where $w=\max_{i}|W_i|$: the third step of
the procedure could involve all the partial outcomes defined on
$W_i$. If the program bounds $w$, the procedure become linear in
$N$. Note that if the LCP-theory corresponds to a CP-net or a
CP-theory then $|W_i|=1~ \forall i$ and the algorithm coincides with the sweep-forward
procedure for CP-nets, and the procedure introduced in
\cite{Wilson04extendingcp-nets} for CP-theories. 

%

We can use this procedure also to compute the optimal completion of a
given partial outcome, simply considering the partial outcome as
pre-assigned values for a subset of features.  

Recursive LCP-theories may require $m$ dominance queries, where $m$ is
the number of dominance goals in the body of the input preference
rules. Since we use tabling, the time cost is amortized over all calls from
the problem-solver (in lieu of space). With this change, the procedure
described above can be used. Because of the dominance queries, the
optimality procedure is PSpace-complete. 

\subsection{Decreasing complexity using evidence specification}
We note that if evidence for some variables is given, the resulting
simplified LCP theory may have the structure of one of the special
cases discussed above, and hence optimality and dominance queries may
be answered using specialized, polynomial procedures.  To this end,
the implementation needs to maintain a dynamic dependency graph that
ignores features for which values have been provided.

 \section{Experimental evaluation}\label{sec:implementation}

We have developed a compiler for LCP theories, also called LCP. The
compiler and associated tooling will be made avaiable on Github as an
open source project under the Eclipse Public Licence.

The compiler reads an LCP-theory, builds the dependency graph and
checks whether it represents an acyclic CP-net. If so, it performs a
linearization of the dependency graph, and produces a pre-digested
representation of the theory. Otherwise it emits the clauses unchanged
so that the standard default (tabled) algorithms optimality,
consistency and dominance can be used.

In more detail, the compiler captures (a linearization of) the dependency order in a clause 

\begin{tabular}{l}
{\footnotesize
  \code{dependency([}$a_1, a_2,\ldots,a_N$\code{]).}}
\end{tabular}

\noindent where $a_i \in 1\ldots N$ (features are implemented as
Prolog integers), and if 
$a_j$ depends on $a_i$, then $i < j$. Suppose $a_p$
depends on $a_{i_1}, \ldots, a_{i_k}$. 
If the input LCP program specifies that if each of the
features $a_{i_1}, \ldots, a_{i_k}$ had values $x_{i_1}, \ldots, 
x_{i_k}$ respectively, then the known order of values of $a_p$ is given by
(best) $w_1, \ldots, w_r$ (worst), then the compiler emits the fact

\begin{tabular}{l}
{\footnotesize
 \code{preference([}$x_{p-1},\ldots, x_1$\code{], [}$w_1, \ldots, w_r$\code{]).}
}
\end{tabular}

\noindent with $x_{i_1},\ldots, x_{i_k}$ as constant, 
and the remaining $x_i$ as unique variables (occur only once in the
clause). Thus we use Prolog
unification to select the correct preference clause to use, given the
current partial outcome \code{[v\_p, \ldots, v\_1]} specifying values
for the first \code{p} attributes (in reverse dependency order). 

The following code for \code{optimize/1} uses these clauses and
implements {\em Acyclic-LCP-Opt}: 
{\footnotesize
\begin{verbatim}
optimize(O) :- 
   O=o(_,_), dependency(D), reorder(D, O, AList), optimize_a([], AList).
reorder([], _, []).
reorder([X| Xs], O, [V | Vs]) :- arg(X, O, V), reorder(Xs, O, Vs).
optimize_a(_, []).
optimize_a(Upargs, [Xs|R]) :-
   select(Upargs, Xs), append(Xs, Upargs, Upargs1), optimize_a(Upargs1, R).
select(Us, []).
select(Us, [X | R]) :- nonvar(X), select(Us, R).
select(Us, [X | R]) :- var(X), preference(Upargs, [X|_]), select(Us ,R).
\end{verbatim}
}
We have also implemented a CP-net generator. 
Generating CP-nets i.i.d. is non-trivial \cite{tomallen} and therefore we use an approximation
method that randomly generates acyclic
CP-nets with $N$ features, given a maximum number of dependencies for each
feature. We consider a fixed ordering $X_1,\ldots ,X_N$ of features. We also
take as input the maximum in-degree for each feature, $k$. We first
generate the acyclic dependency graph. For each feature $X_i$, we
randomly choose its in-degree $d \in 0..\min\{k,{i-1}\}$. 
Next, we randomly choose $d$ parents from the features
$\{X_1, \ldots, X_{i-1} \}$. When the graph is built, we fill in the CP-tables
choosing randomly one element of the domain (since the domain is
binary). 
The resulting CP-net is written out as an LCP theory, using
XSB Version 3.5.0 syntax \cite{xsbonline}. 

We have run two different kinds of experiment: in the first
with a fixed upper bound for the number of dependencies for
each feature, we varied the number of features from $5$ to $200$ and
measure running time for optimality queries. In the
second experiment, fixing the number of features, we varied the upper
bound of dependencies from $1$ to $10$. In both experiments we asked for
the optimal outcome $100$ times and then we computed the average elapsed
time to output the result. 

Figure \ref{fig:fig1} shows the results for the experiment where the
upper bound for the number of dependencies is fixed to $6$. The elapsed
time to compute the optimal outcome grows quadratically in the number
of features. This is in line with our results as summarized in Table
\ref{table:complexity1}. (The runtime is not linear because each step
involves checking for the value of parents using unification on $O(N)$
terms.)  

\begin{figure}[t] 
\centering 
\includegraphics[height=2.5in]{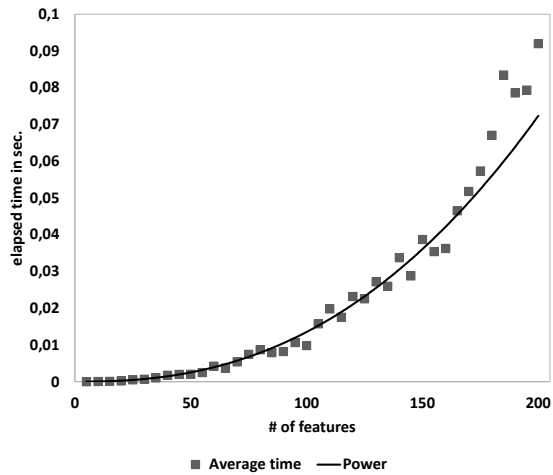}
\caption{Optimal outcome performances.}
\label{fig:fig1}
\end{figure}

\section{Conclusion and future work}\label{sec:conclusion}
We have presented a new framework for conditional preferences, based
on expressing preferences using Datalog. We have shown how dominance,
consistency and optimality queries can be formulated directly in
Datalog and implemented in modern tabled Prolog systems such as XSB
Prolog. We have also analyzed the complexity of the different
algorithms, developed efficient procedures for some common use cases,
and implemented a translator that exploits these algorithms.

Much work lies ahead. Table~\ref{table:complexity1} contains some open complexity
questions. We hope to exploit Datalog theory to develop more efficient
special cases. We hope to use the implemented LCP system in real-life
applications to determine the adequacy of the LCP system. 

\paragraph{Acknowledgements.} We gratefully acknowledge conversations
with David S.~Warren and Terrance Swift. These conversations led to a
small code-change in the XSB compiler to better support LCP theories.
We also acknowledge conversations with Francesca Rossi.

\newpage

\bibliographystyle{alpha}

\begin{thebibliography}{BBD{\etalchar{+}}04b}

\bibitem[AGM14]{tomallen}
T.E. Allen, J.~Goldsmith, and N.~Mattei.
\newblock Counting, ranking, and randomly generating {CP}-nets.
\newblock In {\em In Proceedings of the 8th Multidisciplinary Workshop on
  Advances in Preference Handling (MPREF)}, 2014.

\bibitem[BBD{\etalchar{+}}04a]{cp+constraints}
C.~Boutilier, I.~Brafman, C.~Domshlak, H.~Hoos, and D~Poole.
\newblock Preference-based {C}onstrained {O}ptimization with {CP}-nets.
\newblock {\em Computational Intelligence}, 20(2):137--157, 2004.

\bibitem[BBD{\etalchar{+}}04b]{boutilier2004cp}
C.~Boutilier, R.I. Brafman, C.~Domshlak, H.H. Hoos, and D.~Poole.
\newblock {CP}-nets: {A} tool for representing and reasoning with conditional
  ceteris paribus preference statements.
\newblock {\em Journal of Artificial Intelligence Research}, 21:135--191, 2004.

\bibitem[BFMZ13]{francesi}
D.~Bigot, H.~Fargier, J.~Mengin, and B.~Zanuttini.
\newblock Probabilistic conditional preference networks.
\newblock In {\em Proc. of the 29th International Conference on Uncertainty in
  Artificial Intelligence (UAI)}, 2013.

\bibitem[CGT89]{Ceri:1989:YAW:627272.627357}
S.~Ceri, G.~Gottlob, and L.~Tanca.
\newblock What you always wanted to know about {D}atalog (and never dared to
  ask).
\newblock {\em IEEE Trans. on Knowl. and Data Eng.}, 1(1):146--166, March 1989.

\bibitem[DB02]{domshlak2002cp}
C.~Domshlak and R.I. Brafman.
\newblock {CP}-nets: {R}easoning and consistency testing.
\newblock In {\em Proc. 8th International Conference on Principles of Knowledge
  Representation and Reasoning (KR)}, 2002.

\bibitem[DEGV01]{Dantsin:2001:CEP:502807.502810}
E.~Dantsin, T.~Eiter, G.~Gottlob, and A.~Voronkov.
\newblock Complexity and expressive power of logic programming.
\newblock {\em ACM Comput. Surv.}, 33(3):374--425, September 2001.

\bibitem[DPR{\etalchar{+}}06]{hard+soft+cpnets}
C.~Domshlak, S.~Prestwich, F.~Rossi, K.~Venable, and T.~Walsh.
\newblock Hard and soft constraints for reasoning about qualitative conditional
  preferences.
\newblock {\em J. Heuristics}, 12(4-5):263--285, 2006.

\bibitem[DRVW03]{domshlak03reasoning}
C.~Domshlak, F.~Rossi, K.B. Venable, and T.~Walsh.
\newblock Reasoning about soft constraints and conditional preferences:
  complexity results and approximation techniques.
\newblock In {\em Proc. of the 18th International Joint Conference on
  Artificial Intelligence (IJCAI)}, 2003.

\bibitem[FV99]{Feder:1999:CSM:298483.298498}
T.~Feder and M.~Vardi.
\newblock The {C}omputational {S}tructure of {M}onotone {M}onadic {SNP} and
  {C}onstraint {S}atisfaction: A {S}tudy {T}hrough {D}atalog and {G}roup
  {T}heory.
\newblock {\em SIAM J. Comput.}, 28(1):57--104, February 1999.

\bibitem[GLTW08]{goldsmith2008computational}
J.~Goldsmith, J.~Lang, M.~Truszczynski, and N.~Wilson.
\newblock The {C}omputational {C}omplexity of {D}ominance and {C}onsistency in
  {CP}-nets.
\newblock {\em Journal of Artificial Intelligence Research}, 33(1):403--432,
  2008.

\bibitem[GP03]{Gottlob:2003:CSD:794095.794102}
G.~Gottlob and C.~Papadimitriou.
\newblock On the complexity of single-rule datalog queries.
\newblock {\em Inf. Comput.}, 183(1):104--122, May 2003.

\bibitem[KKR95]{Kanellakis199526}
P.C. Kanellakis, G.M. Kuper, and P.Z. Revesz.
\newblock Constraint query languages.
\newblock {\em Journal of Computer and System Sciences}, 51(1):26 -- 52, 1995.

\bibitem[MMP{\etalchar{+}}13]{MaranInfluencedCPNets}
A.~Maran, N.~Maudet, M.~S. Pini, F.~Rossi, and K.~B. Venable.
\newblock A {F}ramework for {A}ggregating {I}nfluenced {CP}-nets and {I}ts
  {R}esistance to {B}ribery.
\newblock In {\em Proceedings of AAAI-27}, pages 668--674, 2013.

\bibitem[PRVW04]{Prestwich04constrainedcpnets}
S.~Prestwich, F.~Rossi, K.~B. Venable, and T.~Walsh.
\newblock Constrained {CP}-nets.
\newblock In {\em in Proceedings of CSCLP'04}, 2004.

\bibitem[RVW04]{rossi2004mcp}
F.~Rossi, K.B. Venable, and T.~Walsh.
\newblock {mCP} nets: representing and reasoning with preferences of multiple
  agents.
\newblock In {\em Proc. of the 19th AAAI Conference on Artificial Intelligence
  (AAAI)}, 2004.

\bibitem[SW10]{Swift:2010:TAS:1888743.1888771}
T.~Swift and D.~S. Warren.
\newblock Tabling with answer subsumption: Implementation, applications and
  performance.
\newblock In {\em Proceedings of the 12th European Conference on Logics in
  Artificial Intelligence}, JELIA'10, pages 300--312, Berlin, Heidelberg, 2010.
  Springer-Verlag.

\bibitem[SW12a]{tabling-xsb}
T.~Swift and D.~S. Warren.
\newblock {XSB}: Extending {P}rolog with {T}abled {L}ogic {P}rogramming.
\newblock {\em Theory Pract. Log. Program.}, 12(1-2):157--187, January 2012.

\bibitem[SW12b]{xsbonline}
T.~Swift and D.~S. Warren.
\newblock {XSB} home page.
\newblock \url{http://http://xsb.sourceforge.net/}, 2012.

\bibitem[Tom98]{Toman:1998:MEC:593226.593267}
D.~Toman.
\newblock Memoing {E}valuation for {C}onstraint {E}xtensions of {D}atalog.
\newblock {\em Constraints}, 2(3/4):337--359, January 1998.

\bibitem[Var82]{vardi}
M.~Vardi.
\newblock The complexity of relational query languages (extended abstract.
\newblock In {\em Proceedings of the Fourteenth Annual ACM Symposium on Theory
  of Computing (STOC Õ82}, pages 137--146, 1982.

\bibitem[Wil04]{Wilson04extendingcp-nets}
N.~Wilson.
\newblock Extending {CP}-{N}ets with {S}tronger {C}onditional {P}reference
  {S}tatements.
\newblock In {\em Proceedings of AAAI-04}, pages 735--741, 2004.

\bibitem[Wil09]{wilson09}
N.~Wilson.
\newblock Efficient {I}nference for {E}xpressive {C}omparative {P}reference
  {L}anguages.
\newblock In {\em Proceedings of IJCAI-09}, 2009.

\end{thebibliography}
\newcommand{\etalchar}[1]{$^{#1}$}

\end{document}